\documentclass{llncs}
\usepackage{makeidx}  
\usepackage{amssymb}
\usepackage{amsmath}
\usepackage{graphicx}
\usepackage[ruled]{algorithm2e}

\usepackage{url}
\urldef{\mailsa}\path|gergely.neu@gmail.com, bartok@inf.ethz.ch|
\newcommand{\keywords}[1]{\par\addvspace\baselineskip
\noindent\keywordname\enspace\ignorespaces#1}

\usepackage{bm}
\usepackage[disable]{todonotes}
\usepackage{natbib}

\definecolor{PalePurp}{rgb}{0.66,0.57,0.66}
\newcommand{\todoG}[1]{\todo[color=PalePurp!30]{#1}}

\newcommand{\loss}{\ell}
\newcommand{\hloss}{\hat{\ell}}

\newcommand{\real}{\mathbb{R}}
\newcommand{\Sw}{\mathcal{S}}

\newcommand{\II}[1]{\mathbb{I}\left\{#1\right\}}
\newcommand{\PP}[1]{\mathbb{P}\left[#1\right]}
\newcommand{\EE}[1]{\mathbb{E}\left[#1\right]}

\newcommand{\PPc}[2]{\mathbb{P}\left[#1\left|#2\right.\right]}

\newcommand{\PPcc}[2]{\mathbb{P}\left[\left.#1\right|#2\right]}

\newcommand{\EEc}[2]{\mathbb{E}\left[#1\left|#2\right.\right]}
\newcommand{\EEcc}[2]{\mathbb{E}\left[\left.#1\right|#2\right]}

\newcommand{\ev}[1]{\left\{#1\right\}}
\newcommand{\F}{\mathcal{F}}

\newcommand{\tI}{\widetilde{I}}
\newcommand{\tp}{\tilde{p}}
\newcommand{\tq}{\tilde{q}}

\renewcommand{\th}{\ensuremath{^{\mathrm{th}}}}
\DeclareMathOperator*{\argmin}{\arg\min}
\newcommand{\hL}{\wh{L}}
\newcommand{\p}{p}

\newcommand{\bV}{\boldsymbol{V}}
\newcommand{\be}{\boldsymbol{e}}
\newcommand{\bK}{\boldsymbol{K}}

\newcommand{\bv}{\boldsymbol{v}}

\newcommand{\bz}{\boldsymbol{z}}
\newcommand{\bloss}{\bm\ell}
\newcommand{\bL}{\boldsymbol{L}}

\newcommand{\tZ}{\widetilde{\bZ}}
\newcommand{\tV}{\widetilde{\bV}}
\newcommand{\hbl}{\hat{\bloss}}

\newcommand{\hbL}{\wh{\bL}}

\newcommand{\bp}{\boldsymbol{p}}
\newcommand{\bZ}{\boldsymbol{Z}}

\newcommand{\wh}{\widehat}
\newcommand{\wt}{\widetilde}

\newcommand{\fpl}{{FPL}}
\newcommand{\gr}{{GR}}
\newcommand{\exph}{{Exp3}}
\newcommand{\green}{{Green}}

\begin{document}

\mainmatter  

\title{An Efficient Algorithm for Learning with Semi-Bandit Feedback}

\titlerunning{An efficient algorithm for learning with semi-bandit feedback}

%
%
\author{Gergely Neu\inst{1,2} \and G\'abor Bart\'ok\inst{3}}
\authorrunning{Gergely Neu \and G\'abor Bart\'ok}

\institute{Department of Computer Science and Information Theory\\
       Budapest University of Technology and Economics\\
\and MTA SZTAKI Institute for Computer Science and Control\\
       \email{gergely.neu@gmail.com}
\and Department of Computer Science, ETH Z\"urich\\\email{bartok@inf.ethz.ch}}

%
%

\maketitle

\begin{abstract}
We consider the problem of online combinatorial optimization under semi-bandit feedback. The goal of the learner is to sequentially select its actions from a combinatorial decision set so as to minimize its cumulative loss. We propose a learning algorithm for this problem based on combining the Follow-the-Perturbed-Leader (\fpl) prediction method with a novel loss estimation procedure called Geometric Resampling (\gr). Contrary to previous solutions, the resulting algorithm can be efficiently implemented for any decision set where efficient offline combinatorial optimization is possible at all. Assuming that the elements of the decision set can be described with $d$-dimensional binary vectors with at most $m$ non-zero entries, we show that the expected regret of our algorithm after $T$ rounds is $O(m\sqrt{dT\log d})$. As a side result, we also improve the best known regret bounds for \fpl\, in the full information setting to $O(m^{3/2}\sqrt{T\log d})$, gaining a factor of $\sqrt{d/m}$ over previous bounds for this algorithm.
\keywords{Follow-the-perturbed-leader, bandit problems, online learning, combinatorial optimization}
\end{abstract}

\section{Introduction}
In this paper, we consider a special case of online linear optimization known as online combinatorial optimization (see Figure~\ref{fig:protocol}). In every time step $t=1,2,\dots,T$ of this sequential decision problem, the learner chooses an \emph{action} $\bV_t$ from the finite action set $\Sw\subseteq\ev{0,1}^d$, where $\left\|\bv\right\|_1\le m$ holds for all $\bv\in\Sw$. At the same time, the environment fixes a loss vector $\bloss_t\in[0,1]^d$ and the learner suffers loss $\bV_t^\top\bloss_t$. We allow the loss vector $\bloss_t$ to depend on the previous decisions $\bV_1,\dots,\bV_{t-1}$ made by the learner, that is, we consider \emph{non-oblivious} environments. The goal of the learner is to minimize the cumulative loss
$\sum_{t=1}^T\bV_t^\top\bloss_t$.
Then, the performance of the learner is measured in terms of the total expected \emph{regret}
\begin{equation}\label{eq:regret}
R_T= \max_{\bv\in\Sw} \EE{\sum_{t=1}^T\left(\bV_t - \bv\right)^\top\bloss_t} = \EE{\sum_{t=1}^T\bV_t^\top\bloss_t} - \min_{\bv\in\Sw} \EE{\sum_{t=1}^T\bv^\top\bloss_t},
\end{equation}
Note that, as indicated in Figure~\ref{fig:protocol}, the learner chooses its actions randomly, hence the expectation.
\begin{figure}[t]
\fbox{
\begin{minipage}{\textwidth}
{\bfseries Parameters}: set of decision vectors $\Sw = \left\{\bv(1),\bv(2),\dots,\bv(N)\right\}\subseteq\ev{0,1}^d$,
number of rounds $T$; \\
{\bfseries For all $t=1,2,\dots,T$, repeat}
\begin{enumerate}
\item The learner chooses a probability distribution $\bp_t$ over $\left\{1,2,\dots,N\right\}$.
\item The learner draws an action $I_t$ randomly according to $\bp_t$. Consequently, the learner plays decision vector $\bV_t = \bv(I_t)$.
\item The environment chooses loss vector $\bloss_{t}$.
\item The learner suffers loss $\bV_t^\top\bloss_{t}$.
\item The learner observes some feedback based on $\bloss_t$ and $\bV_t$.
\end{enumerate}
\end{minipage}
}
\caption{The protocol of online combinatorial optimization.}
\label{fig:protocol}
\end{figure}

The framework described above is general enough to accommodate a number of interesting problem instances such as path planning, ranking and matching problems, finding minimum-weight spanning trees and cut sets. Accordingly, different versions of this general learning problem have drawn considerable attention in the past few years. These versions differ in the amount of information made available to the learner after each round $t$. In the simplest setting, called the \emph{full-information} setting, it is assumed that the learner gets to observe the loss vector $\bloss_t$ regardless of the choice of $\bV_t$. However, this assumption does not hold for many practical applications, so it is more interesting to study the problem under \emph{partial information}, meaning that the learner only gets some limited feedback based on its own decision. In particular, in some problems it is realistic to assume that the learner observes the vector $\left(V_{t,1}\loss_{t,1},\dots,V_{t,d}\loss_{t,d}\right)$, where $V_{t,i}$ and $\loss_{t,i}$ are the $i\th$ components of the vectors $\bV_t$ and $\bloss_t$, respectively. This information scheme is called \emph{semi-bandit} information. An even more challenging variant is the \emph{full bandit} scheme where all the learner observes after time $t$ is its own loss $\bV_t^\top \bloss_t$.

The most well-known instance of our problem is the (adversarial) \emph{multi-armed bandit} problem considered in the seminal paper of \citet{auer2002bandit}: in each round of this problem, the learner has to select one of $N$ arms and minimize regret against the best fixed arm, while only observing the losses of the chosen arm. In our framework, this setting corresponds to setting $d=N$ and $m=1$, and assuming either full bandit or semi-bandit feedback. Among other contributions concerning this problem, \citeauthor{auer2002bandit} propose an algorithm called \exph\,(Exploration and Exploitation using Exponential weights) based on constructing loss estimates $\hloss_{t,i}$ for each component of the loss vector and playing arm $i$ with probability proportional to $\exp(-\eta \sum_{s=1}^{t-1} \hloss_{s,i})$ at time $t$ ($\eta>0$)\footnote{In fact, \citeauthor{auer2002bandit} mix the resulting distribution with a uniform distribution over the arms with probability $\gamma>0$. However, this modification is not needed when one is concerned with the total expected regret, see e.g., \citet[Chapter~15]{BaPaSzSz11-online}.}. This algorithm is known as the Exponentially Weighted Average (EWA) forecaster in the full information case. Besides proving that the total expected regret of this algorithm is $O(\sqrt{NT\log N})$,  \citeauthor{auer2002bandit} also provide a general lower bound of $\Omega(\sqrt{NT})$ on the regret of any learning algorithm on this particular problem. This lower bound was later matched by the Implicitly Normalized Forecaster (INF) of \citet{audibert09minimax,audibert10inf} by using the same loss estimates in a more refined way.

The most popular example of online learning problems with actual combinatorial structure is the shortest path problem first considered by \citet{TW03} in the full information scheme.
The same problem was considered by \citet{gyorgy07sp}, who proposed an algorithm that works with semi-bandit information. Since then, we have come a long way in understanding the ``price of information'' in online combinatorial optimization---see \citet{audibert12regret} for a complete overview of results concerning all of the presented information schemes. The first algorithm directly targeting general online combinatorial optimization problems is due to \citet{KWK10}: their method named Component Hedge guarantees an optimal regret of $O(m\sqrt{T\log d})$ in the full information setting. In particular, this algorithm is an instance of the more general algorithm class known as Online Stochastic Mirror Descent (OSMD) or Follow-The-Regularized-Leader (FTRL) methods. \citet{audibert12regret} show that OSMD/FTRL-based methods can also be used for proving optimal regret bounds of $O(\sqrt{mdT})$ for the semi-bandit setting. Finally, \citet{bubeck12bandit} show that the natural extension of the EWA forecaster (coupled with an intricate exploration scheme) can be applied to obtain a $O(m^{3/2}\sqrt{dT\log d})$ upper bound on the regret when assuming full bandit feedback. This upper bound is off by a factor of $\sqrt{m\log d}$ from the lower bound proved by \citet{audibert12regret}. For completeness, we note that the EWA forecaster attains a regret of $O(m^{3/2}\sqrt{T\log d})$ in the full information case and $O(m\sqrt{dT\log d})$ in the semi-bandit case.

While the results outlined above suggest that there is absolutely no work left to be done in the full information and semi-bandit schemes, we get a different picture if we restrict our attention to \emph{computationally efficient} algorithms. First, methods based on exponential weighting of each decision vector can only be efficiently implemented for a handful of decision sets $\Sw$---see \citet{KWK10} and \citet{CL12} for some examples. Furthermore, as noted by \citet{audibert12regret}, OSMD/FTRL-type methods can be efficiently implemented by convex programming if the convex hull of the decision set can be described by a polynomial number of constraints. Details of such an efficient implementation are worked out by \citet{suehiro12submodular}, whose algorithm runs in $O(d^6)$ time, which can still be prohibitive in practical problems. While \citet{KWK10} list some further examples where OSMD/FTRL can be implemented efficiently, we conclude that results concerning general efficient methods for online combinatorial optimization are lacking for (semi or full) bandit information problems.

The Follow-the-Perturbed-Leader (\fpl) prediction method (first proposed by \citet{Han57} and later rediscovered by \citet{KV05}) method offers a computationally efficient solution for the online combinatorial optimization problem given that the \emph{static} combinatorial optimization problem $\min_{\bv\in\Sw} \bv^\top\bloss$ admits computationally efficient solutions for any $\bloss\in\real^{d}$. \fpl, however, is usually relatively overlooked due to many ``reasons'', some of them listed below:
\begin{itemize}
\item The best known bound for \fpl\, in the full information setting is $O(m\sqrt{dT})$, which is worse than the bounds for both EWA and OSMD/FTRL. However, this result was recently improved to $O(m^2\sqrt{T\mathop{\textmd{polylog}} d})$ by \citet{devroye13rwalk}.
\item It is commonly believed that the standard proof techniques for \fpl\, do not apply directly against adaptive adversaries (see, e.g, the comments of \citet[Section~2.3]{audibert12regret} or \citet[Section~4.3]{CBLu06:Book}). On the other hand, a direct analysis for non-oblivious adversaries is given by \citet{Pol05} in the multi-armed bandit setting.
\item Considering bandit information, no efficient \fpl-style algorithm is known to achieve a regret of $O(\sqrt{T})$. \citet{AweKlein04} and \citet{McMaBlu04} proposed \fpl-based algorithms for learning with full bandit feedback in shortest path problems, and proved $O(T^{2/3})$ bounds on the regret \eqref{eq:regret}. \citet{Pol05} proved bounds of $O(\sqrt{NT\log N})$ in the $N$-armed bandit setting, however, the proposed algorithm requires $O(T^2)$ computations per time step.
\end{itemize}
In this paper, we offer an \emph{efficient \fpl-based algorithm for regret minimization under semi-bandit feedback}. Our approach relies on a novel method for estimating components of the loss vector. The method, called \emph{geometric resampling} (\gr), is based on the idea that the reciprocal of the probability of an event can be estimated by measuring the reoccurrence time. We show that the regret of \fpl\, coupled with \gr\, attains a regret of $O(m\sqrt{dT\log d})$ in the semi-bandit case.  To the best of our knowledge, our algorithm is the first computationally efficient learning algorithm for this learning problem. As a side result, we also improve the regret bounds of \fpl\, in the full information setting to $O(m^{3/2}\sqrt{T\log d})$, that is, we close the gaps between the performance bounds of \fpl\, and EWA under both full information and semi-bandit feedback.

%

\section{Loss estimation by geometric resampling}\label{sec:resamp}
For a gentle start, consider the problem of regret minimization in $N$-armed bandits where $d=N$, $m=1$ and the learner has access to the basis vectors $\ev{\be_i}_{i=1}^d$. In each time step, the learner specifies a distribution $\bp_t$ over the arms, where $p_{t,i} = \PPcc{I_t=i}{\F_{t-1}}$, where $\F_{t-1}$ is the history of the learner's observations and choices up to the end of time step $t-1$.
Most bandit algorithms rely on feeding some loss estimates to a black-box prediction algorithm.
It is commonplace to consider loss estimates of the form
\begin{equation}\label{eq:oldest}
\hloss_{t,i} = \frac{\loss_{t,i}}{\p_{t,i}} \II{I_t=i},
\end{equation}
where $\p_{t,i} = \PPc{I_t = i}{\F_{t-1}}$, where $\F_{t-1}$ is the history of observations and internal random variables used by the algorithm up to time $t-1$. It is very easy to show that $\hloss_{t,i}$ is an unbiased estimate of the loss $\loss_{t,i}$ for all $t,i$ such that $p_{t,i}$ is positive. For all other $i$ and $t$, $\EEcc{\hloss_{t,i}}{\F_{t-1}} = 0 \le \loss_{t,i}$.

To our knowledge, all existing bandit algorithms utilize some version of the loss estimates described above. While for many algorithms (such as the \exph\, algorithm of \cite{auer2002bandit} and the \green\, algorithm of \cite{allenberg06hannan}), the probabilities $p_{t,i}$ are readily available and the estimates \eqref{eq:oldest} can be computed efficiently, this is not necessarily the case for all algorithms. In particular, \fpl\, is notorious for not being able to handle bandit information efficiently since the probabilities $p_{t,i}$ cannot be expressed in closed form. To overcome this difficulty, we propose a different loss estimate that can be efficiently computed \emph{even when $p_{t,i}$ is not available for the learner}.

The estimation procedure executed after each time step $t$ is described below.

\vspace{.25cm}
\makebox[\textwidth][c]{
\fbox{
\begin{minipage}{\textwidth/2}
\begin{enumerate}
\item The learner draws $I_t\sim \bp_t$.
\item For $n=1,2,\dots$
\begin{enumerate}
\item Let $n\leftarrow n+1$.
\item Draw $I'_t(n) \sim \bp_t$.
\item If $I'_t(n) = I_t$, break.
\end{enumerate}
\item Let $K_t = n$.
\end{enumerate}
\end{minipage}}}
\vspace{.25cm}

\noindent Clearly, $K_t$ is a geometrically distributed random variable given $I_t$ and $\F_{t-1}$. Consequently, we have $\EEc{K_t}{\F_{t-1},I_t} = 1/\p_{t,I_t}$.
We use this property to construct the estimates
\begin{equation}\label{eq:newest}
\hloss_{t,i} = \loss_{t,i} \II{I_t=i} K_t
\end{equation}
for all arms $i$.
We can easily show that the above estimate is conditionally unbiased whenever $p_{t,i}>0$:
\[
\begin{split}
\EEcc{\hloss_{t,i}}{\F_{t-1}} &= \sum_{j} p_{t,j} \EEcc{\hloss_{t,i}}{\F_{t-1},I_t=j}
\\
&= p_{t,i} \EEc{\loss_{t,i} K_t}{\F_{t-1},I_t=i}
\\
&= p_{t,i} \loss_{t,i} \EEc{K_t}{\F_{t-1},I_t=i}
\\
&= \loss_{t,i}.
\end{split}
\]
Clearly $\EEcc{\hloss_{t,i}}{\F_{t-1}} = 0$ still holds whenever $p_{t,i}=0$.

The main problem with the above sampling procedure is that its worst-case running time is unbounded: while the expected number of necessary samples $K_t$ is clearly $N$, the actual number of samples might be much larger. To overcome this problem, we maximize the number of samples by $M$ and use $\tilde{K_t} = \min\ev{K_t,M}$ instead of $K_t$ in \eqref{eq:newest}. While this capping obviously introduces some bias, we will show later that for appropriate values of $M$, this bias does not hurt the performance too much.
%

\section{An efficient algorithm for learning with semi-bandit feedback}\label{sec:alg}
\begin{algorithm}[t]
\caption{FPL with GR}\label{alg:FPLGR}
\textbf{Input}: $\Sw = \left\{\bv(1),\bv(2),\dots,\bv(N)\right\}\subseteq\ev{0,1}^d$,
 $\eta\in\real^+$, $M\in\mathbb{Z}^{+}$\;
  \textbf{Initialization}: $\hbL(1)=\cdots=\hbL(d)=0$\;
  \For{t=1,\dots,T}
  {
  Draw $\bZ(1),\dots,\bZ(d)$ independently from distribution $\text{Exp}(\eta)$\;
  Choose action $I = \displaystyle\argmin_{i\in\ev{1,2,\dots,N}} \ev{\bv(i)^\top \left(\hbL - \bZ\right)}$\;
  $K(1)=\dots=K(d)=M$\;
  $k=0$; \tcc*[f]{Counter for reoccurred indices}\\
  \For(\tcc*[f]{Geometric Resamplig}){n=1,\dots ,M-1}
  {
    Draw $\bZ'(1),\dots,\bZ'(d)$ independently from distribution $\text{Exp}(\eta)$\;
    $I(n)=\displaystyle\argmin_{i\in\ev{1,2,\dots,N}} \ev{\bv(i)^\top \left(\hbL - \bZ'\right)}$\;
    \For{j=1,\dots,d}
    {
      \If{$v(I(n))(j)=1 \ \&\  K(j)=M$}
      {
       $K(j)=n$\;
       $k=k+1$\;
       \lIf(\tcc*[f]{All indices reoccurred}){$k=\bigl\|v(I)\bigr\|_1$}{break}
      }
    }
  }
  \lFor(\tcc*[f]{Update}){j=1,\dots,d}
  {
    $\hbL(j)=\hbL(j)+K(j)v(I)(j)\ell(j)$
  }
  }
\end{algorithm}
First, we generalize the geometric resampling method for constructing loss estimates in the semi-bandit case. To this end, let $p_{t,i} = \PPc{I_t=i}{\F_{t-1}}$ and $q_{t,j} = \EEc{V_{t,j}}{\F_{t-1}}$. First, the learner plays the decision vector with index $I_t\sim \bp_t$. Then, it draws $M$ additional indices $I_t'(1),I_t'(2),\dots,I_t'(M)\sim \bp_t$ independently of each other and $I_t$. For each $j=1,2,\dots,d$, we define the random variables
\[
K_{t,j} = \min\ev{1\le s \le M: v_j(I'_t(s)) = v_j(I_t)},
\]
with the convention that $\min\ev{\emptyset}=M$.
We define the components of our loss estimates $\hbl_t$ as
\begin{align}\label{eq:combest}
\hloss_{t,j} = K_{t,j} V_{t,j} \loss_{t,j}
\end{align}
for all $j=1,2,\dots,d$.
Since $V_{t,j}$ are nonzero only for coordinates for which $\loss_{t,j}$ is observed, these estimates are well-defined. Letting $\hbL_t = \sum_{s=1}^t \hbl_s$, at time step $t$ the algorithm draws the components of the perturbation vector $\bZ_t$ independently from an exponential distribution with parameter $\eta$ and selects the index
\[
I_t = \argmin_{i\in\ev{1,2,\dots,N}} \ev{\bv(i)^\top \left(\hbL_{t-1} - \bZ_t\right)}.
\]
As noted earlier, the distribution $\bp_t$, while implicitly specified by $\bZ_t$ and the estimated cumulative losses $\hbL_t$, cannot be expressed in closed form for \fpl. However, sampling the indices $I_t'(1),I_t'(2),\dots,I_t'(M)$ can be carried out by drawing additional perturbation vectors $\bZ_t'(1),\bZ_t'(2),\dots,\bZ_t'(M)$ independently from the same distribution as $\bZ_t$. We emphasize that the above additional indices are never actually played by the algorithm, but are only necessary for constructing the loss estimates. We also note that in general, drawing as much as $M$ samples is usually not necessary since the sampling procedure can be terminated as soon as the values of $K_{t,i}$ are fixed for all $i$ such that $V_{t,i}=1$. We point the reader to Section~\ref{sec:runningtime} for a more detailed discussion of the running time of the sampling procedure.

Pseudocode for the algorithm can be found in Algorithm~\ref{alg:FPLGR}.
We start analyzing our method by proving a simple lemma on the bias of the estimates.
\begin{lemma}\label{lem:bias}
For all $j\in\ev{1,2,\dots,d}$ and $t=1,2,\dots,T$ such that $q_{t,j}>0$, the loss estimates \eqref{eq:combest} satisfy
\[
\EEcc{\hloss_{t,j}}{\F_{t-1}} = \left(1-(1-q_{t,j})^M \right) \loss_{t,j}.
\]
\end{lemma}
\begin{proof}
Fix any $j,t$ satisfying the condition of the lemma. By elementary calculations,
\[
\begin{split}
\EEcc{\hloss_{t,j}}{\F_{t-1}} &= q_{t,j} \loss_{t,j} \EEcc{K_{t,j}}{\F_{t-1},V_{t,j}=1}.
\end{split}
\]
Setting $q = q_{t,j}$ for simplicity, we have
\[
\begin{split}
\EEc{K_{t,j}}{\F_{t-1},V_{t,j}=1} =&
\sum_{n=1}^\infty n (1-q)^{n-1} q - \sum_{n=M}^\infty (n - M) (1-q)^{n-1} q
\\
=&
\sum_{n=1}^\infty n (1-q)^{n-1} q - (1-q)^M \sum_{n=m}^\infty (n - M) (1-q)^{n-M-1} q
\\
=&
\left(1-(1-q)^M \right)\sum_{n=1}^\infty n (1-q)^{n-1} q = \frac{1-(1-q)^M }{q}.
\end{split}
\]
Putting the two together proves the statement.\qed
\end{proof}
The following theorem gives an upper bound on the total expected regret of the algorithm.

\begin{theorem}\label{thm:main}
The total expected regret of \fpl\, with geometric resampling satisfies
\[
R_n \le \frac{m\left(\log d+1\right)}{\eta} + \eta mdT + \frac{dT}{eM}
\]
under semi-bandit information. In particular, setting $\eta = \sqrt{\left(\log d+1\right)/(dT)}$ and $M\ge \sqrt{dT}/(em\sqrt{\log d+1})$, the regret can be upper bounded as
\[
R_n \le 3m\sqrt{dT\left(\log d+1\right)}.
\]
\end{theorem}
Note that regret bound stated above holds for any non-oblivious adversary since the decision $I_t$ only depends on the previous decisions $I_{t-1},\dots,I_1$ through the loss estimates $\hbl_{t-1},\dots,\hbl_1$. While the main ingredients of the proof presented below are rather common (we borrow several ideas from \citet{Pol05}, the proofs of Theorems~3 and~8 of \citet{audibert12regret} and the proof of Corollary~4.5 in \citet{CBLu06:Book}), these elements are carefully combined in our proof to get the desired result.

\begin{proof}
Let $\tZ$ be a perturbation vector drawn independently from the same distribution as $\bZ_1$ and
\[
\tI_t = \argmin_{i\in\ev{1,2,\dots,N}} \ev{\bv(i)^\top\left(\hL_{t} - \tZ\right)}.
\]
In what follows, we will crucially use that $\tV_t=\bv(\tI_t)$ and $\bV_{t+1} = \bv(I_{t+1})$ {are conditionally independent and identically distributed given $\F_{s}$ for any $s\ge t$}. In particular, introducing the notations
\begin{align*}
q_{t,k} &= \EEcc{V_{t,k}}{\F_{t-1}}  & \tq_{t,k} &= \EEcc{\wt{V}_{t,k}}{\F_{t}}\\
p_{t,i} &= \PPcc{I_t=i}{\F_{t-1}}    & \tp_{t,i} &= \PPcc{\tI_t=i}{\F_{t}},
\end{align*}
we will exploit the above property by using $q_{t,k} = \tq_{t-1,k}$ and $p_{t,i} = \tp_{t-1,i}$ numerous times below.

First, let us address the bias of the loss estimates generated by \gr.
By Lemma~\ref{lem:bias}, we have that $\EEcc{\hloss_{t,k}}{\F_{t-1}}\le \loss_{t,k}$ for all $k$ and $t$, and thus $\EEcc{\bv^\top\hbl_t}{\F_{t-1}} \le \bv^\top \bloss_t$
holds for any fixed $\bv\in\Sw$.
Furthermore, we have
\[
\begin{split}
\EEcc{\tV_{t-1}^\top \hbl_t}{\F_{t-1}} &= \EEcc{\sum_{k=1}^d \wt{V}_{t-1,k} \hloss_{t,k}}{\F_{t-1}}
\\
&= \sum_{k=1}^d \tq_{t-1,k} \EEcc{\hloss_{t,k}}{\F_{t-1}}
\\
&= \sum_{k=1}^d \tq_{t-1,k} \left(1-(1-q_{t,k})^M \right) \loss_{t,k},
\end{split}
\]
where we used the fact that $\tV_{t-1}$ is independent of $\hbl_t$ in the second line and Lemma~\ref{lem:bias} in the last line. Now using that $\tq_{t-1,k} = q_{t,k}$ for all $k$ and $t$ and noticing that $\EEcc{\bV_t^\top\loss_t}{\F_{t-1}} = \sum_{k=1}^d q_{t,k} \loss_{t,k}$, we get that
\begin{equation}\label{eq:true}
\EEcc{\bV_t^\top\loss_t}{\F_{t-1}} \le \EEcc{\tV_{t-1}^\top \hbl_t}{\F_{t-1}} + \sum_{i=1}^d q_{t,k} (1-q_{t,k})^M.
\end{equation}
To control $\sum_k q_{t,k} (1-q_{t,k})^M$, note that $q_{t,k} (1-q_{t,k})^M \le q_{t,k} e^{-Mq_{t,k}}$.
Since $f(q) = qe^{-Mq}$ takes its maximum at $q=1/M$, we get
\[
\sum_{k=1}^d q_{t,k} (1-q_{t,k})^M \le \frac{d}{eM}.
\]

Using Lemma~3.1 of \citet{CBLu06:Book} (sometimes referred to as the
{\sl ``be-the-leader''} lemma) for the sequence $\left(\hbl_1 - \tZ,\hbl_2,\dots,\hbl_T\right)$, we obtain
\[
\sum_{t=1}^T \tV_t^\top \hbl_t - \tV_1^\top \tZ \le \sum_{t=1}^T \bv^\top \hbl_t - \bv^\top \tZ
\]
for any $v\in\Sw$. Reordering and taking expectations gives
\begin{equation}\label{eq:btl}
\begin{split}
\EE{\sum_{t=1}^T \left(\tV_t - \bv\right)^\top \hbl_t} \le \EE{\left(\tV_t - \bv\right)^\top\tZ} \le \frac{m\left(\log d+1\right)}{\eta},
\end{split}
\end{equation}
where we used $\EE{\left\|\bZ_{t}\right\|_{\infty}} \le \log d + 1$.
To proceed, we study the relationship between $\tp_{t,i}$ and $\tp_{t-1,i}=p_{t,i}$. To this end, we introduce the ``sparse loss vector'' $\hbl'_t(i)$ with components $\hloss'_{t,k}(i) = v_k(i) \hloss_{t,k}$ and
\[
\tI'_t(i) = \argmin_{i\in\ev{1,2,\dots,N}} \ev{\bv(i)^\top\left(\hbL_{t-1} + \hbl'_t(i) - \tZ\right)}.
\]
Using the notation $\tp'_{t,i} = \PPcc{\tI'_{t}(i)=i}{\F_{t}}$, we show in Lemma~\ref{lem:modified-loss} (stated and proved after the proof of the theorem) that $\tp'_{t,i}\le \tp_{t,i}$.\footnote{Note that a similar trick was used in the proof Corollary~4.5 in \citet{CBLu06:Book}. Also note that this trick only applies in the case of non-negative losses.}
Also, define
\[
J(\bz) = \argmin_{j\in\ev{1,2,\dots,N}} \ev{\bv(j)^\top\left(\hbL_{t-1} - \bz\right)}.
\]
Letting $f(\bz)$ be the density of the perturbations, we clearly have
\[
\begin{split}
\tp_{t-1,i}
&=\int\limits_{\bz\in[0,\infty]^d} \II{J(\bz)=i} f(\bz) \,d\bz
\\
&= e^{\eta \left\|\hbl'_{t}(i)\right\|_1} \int\limits_{\bz\in[0,\infty]^d} \II{J(\bz)=i} f\left(\bz+\hbl'_t(i)\right)  \,d\bz
\\
&= e^{\eta \left\|\hbl'_{t}(i)\right\|_1} \idotsint\limits_{z_i\in[\hbl'_{t,i},\infty]} \II{J\left(\bz-\hbl'_t(i)\right)=i} f(\bz)  \,d\bz
\\
&\le e^{\eta \left\|\hbl'_{t}(i)\right\|_1} \int\limits_{\bz\in[0,\infty]^d} \II{J\left(\bz-\hbl'_t(i)\right)=i} f(\bz)  \,d\bz
\\
&= e^{\eta \left\|\hbl'_{t}(i)\right\|_1} \tp'_{t,i} \le e^{\eta \left\|\hbl'_{t}(i)\right\|_1} \tp_{t,i},
\end{split}
\]
where we used $f(\bz) = \eta\exp(-\eta \|\bz\|_1)$ for $\bz\in[0,\infty]^d$.
Now notice that $\bigl\|\hbl'_{t}(i)\bigr\|_1 = \bv(i)^\top \hbl'_{t}(i) = \bv(i)^\top \hbl_{t}$, which yields
\[
\begin{split}
\tp_{t,i} &\ge \tp_{t-1,i} e^{-\eta \bv(i)^\top \hbl_{t}} \ge \tp_{t-1,i} \left(1 - \eta \bv(i)^\top \hbl_{t}\right).
\end{split}
\]
It follows that
\begin{equation}
\begin{split}\label{eq:pineq}
\EEcc{\tV_{t-1}^\top \hbl_t}{\F_{t}}
&= \sum_{i=1}^N \tp_{t-1,i} \bv(i)^\top \hbl_{t}
\le \sum_{i=1}^N \tp_{t,i} \bv(i)^\top \hbl_{t} + \eta \sum_{i=1}^N \tp_{t-1,i} \left(\bv(i)^\top \hbl_{t}\right)^2
\\
&= \EEcc{\tV_t^\top \hbl_t}{\F_{t}} + \eta \sum_{i=1}^N \tp_{t-1,i} \left(\bv(i)^\top \hbl_{t}\right)^2,
\end{split}
\end{equation}
where we used $\EEcc{\tV_{t-1}}{\F_t} = \EEcc{\tV_{t-1}}{\F_{t-1}}$ in the second equality.
Similarly to the proof of Theorem~8 of \citet{audibert12regret}, the last term can be upper bounded as
\[
\begin{split}
\EEcc{\sum_{i=1}^N \tp_{t-1,i} \left(\bv(i)^\top \hbl_{t}\right)^2}{\F_{t-1}} &=
\EEcc{\sum_{j=1}^d\sum_{k=1}^d \left(\wt{V}_{t-1,j}\hloss_{t,j}\right)\left(\wt{V}_{t-1,k}\hloss_{t,k}\right)}{\F_{t-1}}
\\
&\le
\EEcc{\sum_{j=1}^d\hloss_{t,j}
\sum_{k=1}^d \left(\wt{V}_{t-1,k}K_{t,k}V_{t,k}\loss_{t,k}\right)}{\F_{t-1}}
\\
&\le\EEcc{\sum_{j=1}^d\hloss_{t,j}
\sum_{k=1}^d V_{t,k}\loss_{t,k}}{\F_{t-1}}
\\
&\le m\EEcc{\sum_{j=1}^d\hloss_{t,j}}{\F_{t-1}} \le md,
\end{split}
\]
where we used that $\tV_{t-1}$ is independent of $\bV_{t}$, $\hbl_t$ and $\bK_{t}$, so $\EEcc{\wt{V}_{t-1,k} K_{t,k}}{\F_{t-1}}\le 1$ in the second inequality, and $\EEcc{\hloss_{t,j}}{\F_{t-1}} \le 1$ in the last inequality.
That is, we have proved
\begin{equation}\label{eq:peek}
\EE{\sum_{t=1}^T\tV_{t-1}^\top \hbl_t} \le \EE{\sum_{t=1}^T\tV_t^\top \hbl_t} + \eta md.
\end{equation}
Putting Equations~\eqref{eq:true}, \eqref{eq:btl} and~\eqref{eq:peek} together, we obtain
\[
\EE{\sum_{t=1}^T\left(\bV_t-\bv\right)^\top\bloss_t} \le \frac{m\left(\log d+1\right)}{\eta} + \eta mdT + \frac{dT}{eM}
\]
as stated in the theorem.\qed
\end{proof}
In the next lemma, we prove that $\tp'_{t,i}\le \tp_{t,i}$ holds for all $t$ and $i$. While this statement is rather intuitive, we include its simple proof for completeness.
\begin{lemma}\label{lem:modified-loss}
Fix any $i\in\ev{1,2,\dots,N}$ and any vectors $\bL\in \real^d$ and $\bloss\in[0,\infty)^d$. Furthermore, define the vector $\bloss'$ with components $\loss'_k = v_k(i)\loss_k$ and the perturbation vector $\bZ$ with independent components. Then,
\[
\begin{split}
&\PP{\bv(i)^\top\left(\bL + \bloss' - \bZ\right)\le \bv(j)^\top\left(\bL + \bloss' - \bZ\right)\, \left(\forall j\in\ev{1,2,\dots,N}\right)}
\\
&\qquad\le
\PP{\bv(i)^\top\left(\bL + \bloss - \bZ\right)\le \bv(j)^\top\left(\bL + \bloss - \bZ\right)\, \left(\forall j\in\ev{1,2,\dots,N}\right)}.
\end{split}
\]
\end{lemma}
\begin{proof}
Fix any $\forall j\in\ev{1,2,\dots,N}\setminus i$ and define the vector $\bloss'' = \bloss - \bloss'$.  Define the events
\[
A'_j = \ev{\omega:\,\bv(i)^\top\left(\bL + \bloss' - \bZ\right)\le \bv(j)^\top\left(\bL + \bloss' - \bZ\right)}
\]
and
\[
A_j = \ev{\omega:\,\bv(i)^\top\left(\bL + \bloss - \bZ\right)\le \bv(j)^\top\left(\bL + \bloss - \bZ\right)}.
\]
We have
\[
\begin{split}
A'_j &= \ev{\omega:\,\left(\bv(i)-\bv(j)\right)^{\top}\bZ \ge \left(\bv(i)-\bv(j)\right)^{\top}\left(\bL+\bloss'\right)}
\\
&\subseteq \ev{\omega:\,\left(\bv(i)-\bv(j)\right)^{\top}\bZ \ge \left(\bv(i)-\bv(j)\right)^{\top}\left(\bL+\bloss'\right) - \bv(j)^\top\bloss''}
\\
&= \ev{\omega:\,\left(\bv(i)-\bv(j)\right)^{\top}\bZ \ge \left(\bv(i)-\bv(j)\right)^{\top}\left(\bL+\bloss\right)}
= A_j,
\end{split}
\]
where we used $\bv(i)\bloss''=0$ and $\bv(j)\bloss''\ge 0$. Now, since $A'_j\subseteq A_j$, we have
$
\cap_{j=1}^N A'_j \subseteq \cap_{j=1}^N A_j
$,
thus proving
$
\PP{\cap_{j=1}^N A'_j} \le \PP{\cap_{j=1}^N A_j}
$
as requested.\qed
\end{proof}
\subsection{Running time}\label{sec:runningtime}
Let us now turn our attention to computational issues. As mentioned earlier, since we cut off the number of times we resample the decision vectors, the maximum number of times an arm has to be drawn per time step is $M=\sqrt{dT}$. This implies an $O(T^{3/2}d^{1/2})$ worst-case running time. However, the \emph{expected} running time is much more comforting.
\begin{theorem}
  The expected number of times the algorithm draws an action up to time step $T$ can be upper bounded by $dT$.
\end{theorem}
\begin{proof}
  Let us first modify our algorithm so that it draws even more actions! Let us assume for now that for each coordinate that the original arm had $1$, we keep sampling until we get $1$ in the same coordinate again. Also let us assume that we do not use cutoff. Instead, we always keep sampling until the desired $1$ reoccurs.

  At time step $t$, for a given coordinate $k$ with $1$, the expected number of samples needed is $1/q_{t,k}$, while the probability of coordinate $k$ being $1$ is $q_{t,k}$. Thus, the expected number of samples is
  \begin{flalign*}
    & &\sum_{k=1}^d q_{t,k}\frac{1}{q_{t,k}}=d. & & \Box
  \end{flalign*}
\end{proof}

\section{Improved bounds for learning with full information}
Our proof technique also enables us to tighten the guarantees for \fpl\, in the full information setting. In particular, we consider the algorithm choosing the index
\[
I_t = \argmin_{i\in\ev{1,2,\dots,N}} \ev{\bv(i)^\top \left(\bL_{t-1} - \bZ_t\right)},
\]
where $\bL_t = \sum_{s=1}^t \bloss_t$ and the components of $\bZ_t$ are drawn independently from an exponential distribution with parameter $\eta$.
We state our improved regret bounds concerning this algorithm in the following theorem.
\begin{theorem}\label{thm:fullinfo}
Let $C_T = \sum_{t=1}^T \EE{\bV_t^\top\bloss_t}$. Then the total expected regret of \fpl\, satisfies
\[
R_n \le \frac{m\left(\log d+1\right)}{\eta} + \eta m C_T
\]
under full information. In particular, setting $\eta = \sqrt{\left(\log d+1\right)/(mT)}$, the regret can be upper bounded as
\[
R_n \le 2m^{3/2}\sqrt{T\left(\log d+1\right)}.
\]
\end{theorem}
Note that the above bound can be further tightened if some upper bound $C^*_T\ge C_T$ is available a priori. Once again, these regret bounds hold for any non-oblivious adversary since the decision $I_t$ depends on the previous decisions $I_{t-1},\dots,I_1$ only through the loss vectors $\bloss_{t-1},\dots,\bloss_1$.

\begin{proof}
The statement follows from a simplification of the proof of Theorem~\ref{thm:main} when using $\hbl_t=\bloss_t$. First, identically to Equation~\eqref{eq:btl}, we have
\[
\EE{\sum_{t=1}^T \left(\tV_t - \bv\right)^\top \bloss_t} \le \EE{\left(\tV_t - \bv\right)^\top\tZ} \le \frac{m\left(\log d+1\right)}{\eta}.
\]
Further, it is easy to see that the conditions of Lemma~\ref{lem:modified-loss} are satisfied and, similarly to Equation~\eqref{eq:pineq}, we also have
\[
\begin{split}
\EE{\tV_{t-1}^\top \bloss_t}
&\le
\EE{\tV_t^\top \bloss_t} + \eta \sum_{i=1}^N \tp_{t-1,i} \left(\bv(i)^\top \bloss_{t}\right)^2
\\
&\le
\EE{\tV_t^\top \bloss_t} + \eta m \sum_{i=1}^N \tp_{t-1,i} \bv(i)^\top \bloss_{t}.
\end{split}
\]
Using that $\bV_t$ and $\tV_{t-1}$ have the same distribution, we obtain the statement of the theorem.\qed
\end{proof}

\section{Conclusions and open problems}
In this paper, we have described the first general efficient algorithm for
online combinatorial optimization under semi-bandit feedback. We have proved that the regret of our algorithm is $O(m\sqrt{dT\log d})$ in this setting, and have also shown that \fpl\, can achieve $O(m^{3/2}\sqrt{T\log d})$ in the full information case when tuned properly. While these bounds are off by a factor of $\sqrt{m\log d}$ and $\sqrt{m}$ from the respective minimax results, they exactly match the best known regret bounds for the well-studied Exponentially Weighted Forecaster (EWA). Whether the gaps mentioned above can be closed for \fpl-style algorithms (e.g., by using more intricate perturbation schemes) remains an important open question. Nevertheless, we regard our contribution as a significant step towards understanding the inherent trade-offs between computational efficiency and performance guarantees in online combinatorial optimization and, more generally, in online linear optimization.

The efficiency of our method rests on a novel loss estimation method called geometric resampling (\gr). Obviously, this estimation method is not specific to the proposed learning algorithm. While \gr\, has no immediate benefits for OSMD/FTRL-type algorithms where the probabilities $q_{t,k}$ are readily available, it is possible to think about problem instances where EWA can be efficiently implemented while the values of $q_{t,k}$ are difficult to compute. \todoG{is it?} A particular online learning problem where \gr\, can be useful is the problem of online learning in Markovian decision processes \citep{neu10o-ssp,neu10o-mdp}, where computing $q_{t,k}$ can be computationally expensive when the underlying Markovian environment is complicated. This computational burden can be lightened by using \gr\, if the learner has access to a generative model of the environment.\footnote{In particular, for an MDP with state and action spaces $\mathcal{X}$ and $\mathcal{A}$ and worst-case mixing time $\tau>0$, computing the probabilities $q_{t,k}$ can take up to $O(|\mathcal{X}|^3|\mathcal{A}|)$ time, \gr\, returns sufficiently good estimates by generating $O(|\mathcal{X}||\mathcal{A}|)$ trajectories of length $\tau$. Deciding which approach is more efficient depends on the problem parameters $\mathcal{X}$, $\mathcal{A}$ and $\tau$.}

The most important open problem left is the case of efficient online linear optimization with full bandit feedback. Learning algorithms for this problem usually require that the pseudoinverse of the covariance matrix $P_t = \EEcc{\bV_t \bV_t^\top}{\F_{t-1}}$ is readily available for the learner at each time step (see, e.g., \citealt{McMaBlu04,BDHKRT08,DHK08,CL12,bubeck12bandit}). While for most problems, this inverse matrix cannot be computed efficiently, it can be efficiently approximated by geometric resampling when $P_t$ is positive definite as the limit of the matrix geometric series $\sum_{n=1}^\infty (I-P_t)^n$.  While this knowledge should be enough to construct an efficient \fpl-based method for online combinatorial optimization under full bandit feedback, we have to note that the analysis presented in this paper does not carry through directly in this case: as usual loss estimates might take negative values in the full bandit setting, proving a bound similar to Equation~\eqref{eq:pineq} cannot be performed in the presented manner.
\bibliographystyle{apalike}
\bibliography{ngbib,allbib,predbook}

\begin{thebibliography}{}

\bibitem[Allenberg et~al., 2006]{allenberg06hannan}
Allenberg, C., Auer, P., Gy{\"o}rfi, L., and Ottucs{\'a}k, {\text{Gy}}. (2006).
\newblock Hannan consistency in on-line learning in case of unbounded losses
  under partial monitoring.
\newblock In {\em ALT}, pages 229--243.

\bibitem[Audibert and Bubeck, 2009]{audibert09minimax}
Audibert, J.-Y. and Bubeck, S. (2009).
\newblock Minimax policies for bandits games.
\newblock In {\em COLT 2009}.

\bibitem[Audibert and Bubeck, 2010]{audibert10inf}
Audibert, J.-Y. and Bubeck, S. (2010).
\newblock Regret bounds and minimax policies under partial monitoring.
\newblock {\em Journal of Machine Learning Research}, 11:2785--2836.

\bibitem[Audibert et~al., 2013]{audibert12regret}
Audibert, J.~Y., Bubeck, S., and Lugosi, G. (2013).
\newblock Regret in online combinatorial optimization.
\newblock {\em To appear in Mathematics of Operations Research}.

\bibitem[Auer et~al., 2002]{auer2002bandit}
Auer, P., Cesa-Bianchi, N., Freund, Y., and Schapire, R.~E. (2002).
\newblock The nonstochastic multiarmed bandit problem.
\newblock {\em SIAM J. Comput.}, 32(1):48--77.

\bibitem[Awerbuch and Kleinberg, 2004]{AweKlein04}
Awerbuch, B. and Kleinberg, R.~D. (2004).
\newblock Adaptive routing with end-to-end feedback: distributed learning and
  geometric approaches.
\newblock In {\em Proceedings of the 36th ACM Symposium on Theory of
  Computing}, pages 45--53.

\bibitem[Bartlett et~al., 2008]{BDHKRT08}
Bartlett, P., Dani, V., Hayes, T., Kakade, S., Rakhlin, A., and Tewari, A.
  (2008).
\newblock High probability regret bounds for online optimization.
\newblock In {\em Proceedings of the 21st Annual Conference on Learning Theory
  (COLT)}.

\bibitem[Bart\'ok et~al., 2011]{BaPaSzSz11-online}
Bart\'ok, G., P\'al, D., Szepesv\'ari, {\text{Cs}}., and Szita, I. (2011).
\newblock Online learning.
\newblock Lecture notes, University of Alberta.
\newblock https://moodle.cs.ualberta.ca/file.php/354/notes.pdf.

\bibitem[Bubeck et~al., 2012]{bubeck12bandit}
Bubeck, S., Cesa-Bianchi, N., and Kakade, S.~M. (2012).
\newblock Towards minimax policies for online linear optimization with bandit
  feedback.
\newblock {\em Journal of Machine Learning Research - Proceedings Track},
  23:41.1--41.14.

\bibitem[Cesa-Bianchi and Lugosi, 2006]{CBLu06:Book}
Cesa-Bianchi, N. and Lugosi, G. (2006).
\newblock {\em Prediction, Learning, and Games}.
\newblock Cambridge University Press, New York, NY, USA.

\bibitem[Cesa-Bianchi and Lugosi, 2012]{CL12}
Cesa-Bianchi, N. and Lugosi, G. (2012).
\newblock Combinatorial bandits.
\newblock {\em Journal of Computer and System Sciences}, 78:1404--1422.

\bibitem[Dani et~al., 2008]{DHK08}
Dani, V., Hayes, T., and Kakade, S. (2008).
\newblock {The price of bandit information for online optimization}.
\newblock In {\em Advances in Neural Information Processing Systems (NIPS)},
  volume~20, pages 345--352.

\bibitem[Devroye et~al., 2013]{devroye13rwalk}
Devroye, L., Lugosi, G., and Neu, G. (2013).
\newblock Prediction by random-walk perturbation.
\newblock {\em Accepted to the Twenty-Sixth Conference on Learning Theory}.

\bibitem[Gy\"{o}rgy et~al., 2007]{gyorgy07sp}
Gy\"{o}rgy, A., Linder, T., Lugosi, G., and Ottucs\'{a}k, {\text{Gy}}. (2007).
\newblock The on-line shortest path problem under partial monitoring.
\newblock {\em Journal of Machine Learning Research}, 8:2369--2403.

\bibitem[Hannan, 1957]{Han57}
Hannan, J. (1957).
\newblock Approximation to {B}ayes risk in repeated play.
\newblock {\em Contributions to the theory of games}, 3:97--139.

\bibitem[Kalai and Vempala, 2005]{KV05}
Kalai, A. and Vempala, S. (2005).
\newblock Efficient algorithms for online decision problems.
\newblock {\em Journal of Computer and System Sciences}, 71:291--307.

\bibitem[Koolen et~al., 2010]{KWK10}
Koolen, W., Warmuth, M., and Kivinen, J. (2010).
\newblock Hedging structured concepts.
\newblock In {\em Proceedings of the 23rd Annual Conference on Learning Theory
  (COLT)}, pages 93--105.

\bibitem[McMahan and Blum, 2004]{McMaBlu04}
McMahan, H.~B. and Blum, A. (2004).
\newblock Online geometric optimization in the bandit setting against an
  adaptive adversary.
\newblock In {\em Proceedings of the Eighteenth Conference on Computational
  Learning Theory}, pages 109--123.

\bibitem[Neu et~al., 2010a]{neu10o-ssp}
Neu, G., Gy\"orgy, A., and Szepesv\'ari, {\text{Cs}}. (2010a).
\newblock The online loop-free stochastic shortest-path problem.
\newblock In {\em Proceedings of the Twenty-Third Conference on Computational
  Learning Theory}, pages 231--243.

\bibitem[Neu et~al., 2010b]{neu10o-mdp}
Neu, G., Gy\"orgy, A., Szepesv\'ari, {\text{Cs}}., and Antos, A. (2010b).
\newblock Online {M}arkov decision processes under bandit feedback.
\newblock In {\em Advances in Neural Information Processing Systems 23}, pages
  1804--1812.

\bibitem[Poland, 2005]{Pol05}
Poland, J. (2005).
\newblock {FPL} analysis for adaptive bandits.
\newblock In {\em In 3rd Symposium on Stochastic Algorithms, Foundations and
  Applications (SAGA'05)}, pages 58--69.

\bibitem[Suehiro et~al., 2012]{suehiro12submodular}
Suehiro, D., Hatano, K., Kijima, S., Takimoto, E., and Nagano, K. (2012).
\newblock Online prediction under submodular constraints.
\newblock In {\em Algorithmic Learning Theory}, volume 7568 of {\em Lecture
  Notes in Computer Science}, pages 260--274. Springer Berlin Heidelberg.

\bibitem[Takimoto and Warmuth, 2003]{TW03}
Takimoto, E. and Warmuth, M. (2003).
\newblock Paths kernels and multiplicative updates.
\newblock {\em Journal of Machine Learning Research}, 4:773--818.

\end{thebibliography}

\end{document}